\documentclass{article}

\usepackage{arxiv}

\usepackage[utf8]{inputenc} % allow utf-8 input
\usepackage[T1]{fontenc}    % use 8-bit T1 fonts
\usepackage{hyperref}       % hyperlinks
\usepackage{url}            % simple URL typesetting
\usepackage{booktabs}       % professional-quality tables
\usepackage{amsfonts}       % blackboard math symbols
\usepackage{nicefrac}       % compact symbols for 1/2, etc.
\usepackage{microtype}      % microtypography
\usepackage{lipsum}		% Can be removed after putting your text content
\usepackage{graphicx}
 \hypersetup{
     colorlinks=false,
     linkcolor=blue,
     filecolor=blue,      
     urlcolor=blue,
     pdftitle={Overleaf Example},
     pdfpagemode=FullScreen,
    }
\usepackage[square,sort,comma,numbers]{natbib}
\usepackage{doi}
\usepackage{amsmath, amsthm, amscd, amsfonts, amssymb, graphicx, color}
\usepackage{graphicx}
\usepackage{amsmath}
\usepackage{amsfonts}
\usepackage{mathtools}
\usepackage{algpseudocode}
\usepackage{braket}
\usepackage[linesnumbered,ruled,vlined]{algorithm2e}
\SetKwInput{KwInput}{Input}                % Set the Input
\SetKwInput{KwOutput}{Output} 
\SetKwInput{Kwset}{Set}

\newtheorem{theorem}{Theorem}[section]
\newtheorem{corollary}{Corollary}[section]

\newtheorem{remark}{Remark}[section]
\newtheorem{lemma}{Lemma}[section] 
\title{Modified Step Size for Enhanced Stochastic Gradient Descent: Convergence and Experiments}

\author{% \href{https://orcid.org/0000-0000-0000-0000}
{\includegraphics[scale=0.06]{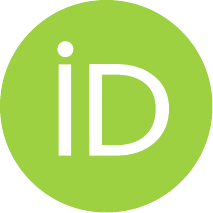}\hspace{1mm}M. Soheil Shamaee}\\
	Department of  Computer Science,\\ 
   Faculty of Mathematical Science, \\
   {University of Kashan}, Kashan, Iran.\\
	\texttt{soheilshamaee@kashanu.ac.ir} \\
	\And
 {\includegraphics[scale=0.06]{orcid.pdf}\hspace{1mm}S. Fathi Hafshejani} \\
	Department of Applied Mathematics,\\
	Shiraz University of Technology,\\
	 Shiraz, Iran\\
	\texttt{s.fathi@sutech.ac.ir} \\
}

\hypersetup{
pdftitle={},
pdfsubject={},
pdfauthor={},
pdfkeywords={},
}

\begin{document}
\maketitle

\begin{abstract}
This paper introduces a novel approach to enhance the performance of the stochastic gradient descent (SGD) algorithm by incorporating a modified decay step size based on $\frac{1}{\sqrt{t}}$. The proposed step size integrates a logarithmic term, leading to the selection of smaller values in the final iterations. Our analysis establishes a convergence rate of $O(\frac{\ln T}{\sqrt{T}})$ for smooth non-convex functions without the Polyak-Łojasiewicz condition. To evaluate the effectiveness of our approach, we conducted numerical experiments on image classification tasks using the FashionMNIST, and CIFAR10 datasets, and the results demonstrate significant improvements in accuracy, with enhancements of $0.5\%$ and  $1.4\%$ observed, respectively, compared to the traditional $\frac{1}{\sqrt{t}}$ step size. The source code can be found at \\\url{https://github.com/Shamaeem/LNSQRTStepSize}
\end{abstract}

% keywords can be removed

\section*{\centerline{ 1. Introduction}}\setcounter{section}{1}\setcounter{theorem}{0}

Stochastic gradient descent (SGD) has a rich historical background, originating from the influential work by Robbins and Monro \cite{robbins1951stochastic}. In the realm of modern machine learning, SGD has emerged as a fundamental optimization algorithm for training deep neural networks (DNNs), which have achieved remarkable performance across diverse domains such as image classification \cite{krizhevsky2017imagenet,krizhevsky2009learning}, object detection \cite{redmon2017yolo9000}, and machine translation \cite{zhang2015deep}.

The selection of an appropriate step size, often referred to as the learning rate, plays a pivotal role in the convergence behavior of SGD.  If the step size value is too large, it can prevent SGD iterations from reaching the optimal point, leading to instability and divergence. On the other hand, excessively small step size values can result in slow convergence and hinder the algorithm's ability to escape suboptimal local minima \cite{mishra2019polynomial}. To tackle these challenges, researchers have proposed various schemes to determine the step size dynamically during the optimization process.

One notable approach is the Armijo line search method, initially introduced by Vaswani et al. \cite{vaswani2019painless}, which provides theoretical guarantees for strong-convex, convex, and non-convex objective functions. Another strategy, proposed by Gower et al. \cite{gower2019sgd}, combines a constant learning rate with a decreasing learning rate schedule. Their algorithm starts with a fixed learning rate and transitions to a decreasing schedule after a specified number of iterations, often determined by the problem's condition number. While this technique ensures convergence for strongly convex functions, it necessitates prior knowledge of the condition number and is not directly applicable to non-convex problems. %In practical settings, decay-based learning rate schedules, such as the popular $\frac{1}{\sqrt{t}}$ decay, have been widely employed in deep neural network training to address non-convex optimization challenges \cite{krizhevsky2017imagenet, huang2017densely}.

Decay step size is a commonly employed strategy in SGD to improve the convergence of optimization algorithms {\cite{Loshchilov,wang2021convergence}}. By gradually reducing the step size over iterations, decay step size methods facilitate finer adjustments in parameter updates, leading to improved convergence behavior and enhanced optimization performance {\cite{Tao,li2021second}.}
Among the various decay step sizes used in SGD, the $\frac{1}{\sqrt{t}}$ step size has been widely used due to its ease of implementation and the significant advantage of not requiring derivative information. For instance, this step size has been recognized for its excellent performance in binary classification, making it an effective choice {\cite{Tao}}. Additionally, it exhibits favorable efficiency in the context of deep neural networks.

During the training of deep neural networks, the use of this step size encounters a fundamental challenge. While the step size value decreases appropriately in the early iterations, it remains excessively large in the final iterations. This leads to the algorithm's inability to converge to the optimal point. As a result, the SGD algorithm with $\frac{1}{\sqrt{t}}$ step size fails to achieve the desired accuracy for deep neural networks. To address this limitation, we propose an enhanced version of the $\frac{1}{\sqrt{t}}$ step size that incorporates the $\ln t$ function into its definition. By introducing this modification, our goal is to improve the accuracy and loss function compared to the original $\frac{1}{\sqrt{t}}$ step size commonly used in SGD.

{
Smith in \cite{Smith} introduced the efficient method of setting the step size, known as the cyclical learning rate. Utilizing cyclical learning rates for training neural networks can yield substantial enhancements in accuracy, eliminating the need for manual tuning and often requiring fewer iterations for convergence \cite{Smith}. Building on this concept, Loshchilov and Hutter presented a warm restart technique for SGD in \cite{Loshchilov}. This approach eliminates the necessity of computing gradient information for adjusting the step size in each iteration. Warm restarts operate by initializing the learning rate to a specific value $\eta_0$ at each restart, scheduling its subsequent decrease \cite{Loshchilov}. Additionally, studies have revealed that warm restarted SGD exhibits significantly improved efficiency, taking notably less time compared to traditional learning rate adjustment strategies \cite{Vrbanc}. Over recent years, a variety of step sizes accompanied by warm restarts have been proposed \cite{mishra2019polynomial,Xu}. Extending the notion of cosine step size, Vrbančič introduced three distinct step sizes accompanied by warm restarts \cite{Vrbanc}.}

Building upon the insights from previous research, we present a novel approach in this work that utilizes a novel step size combined with the warm restarts technique for SGD. The key contributions of this paper can be summarized as follows:
\begin{itemize}
    \item The new step size exhibits a distinct behavior compared to the $\frac{1}{\sqrt{t}}$ step size. By incorporating both $\frac{1}{\sqrt{t}}$ and $\ln t$, the step length gradually decreases in the final iterations, leading to convergence towards the optimal point. The impact of this modification will be demonstrated through the numerical results.
    \item We demonstrate the convergence rate of $O(\frac{\ln T}{\sqrt{T}})$ for smooth non-convex functions, without requiring the Polyak-Łojasiewicz (PL) condition.
    \item %To evaluate the performance of the new step size, we conduct extensive experiments on popular image classification datasets, namely FashionMNIST, and  CIFAR10, the results  demonstrate significant improvements in accuracy, with enhancements of $0.5\%$, and $1.4\%$  observed, respectively, compared to the traditional $\frac{1}{\sqrt{t}}$ step size. 
    { We evaluate the performance of the new step size through extensive experiments on two popular image classification datasets, that is, FashionMNIST and CIFAR10. The results indicate significant improvements in accuracy, with enhancements of $0.5\%$ and $1.4\%$ observed, respectively, when compared to the traditional $\frac{1}{\sqrt{t}}$ step size.}
In addition, we conduct SGD experiments on binary classification tasks using five diverse datasets: a1a, a2a, mushrooms, rcv1, and w1a. The results demonstrate that the new proposed step size consistently outperforms other step sizes in terms of accuracy and loss function.
\end{itemize}
The paper is organized as follows: Section 2 introduces the new step size, providing an overview of its formulation and properties. In Section 3, we analyze the convergence rates of the proposed step size on smooth non-convex functions, demonstrating its impressive $O(\frac{\ln T}{\sqrt{T}})$ convergence rate. Section 4 presents and discusses the numerical results obtained using the new decay step size, highlighting its effectiveness in improving optimization performance. Finally, Section 5 concludes the paper by summarizing the findings and drawing insightful conclusions from our study.\\
{
In this paper, we use the following notational conventions:
The Euclidean norm of a vector is denoted by $\|.\|$. The non-negative  {orthant} and positive orthant of $\Bbb{R}^d$ are denoted by $\Bbb{R}_+^d$ and $\Bbb{R}_{++}^d$, respectively. We use the notation $f(t)=O(g(t))$ to indicate that there exists a positive constant $\omega$ such that $f(t)\leq \omega g(t)$ for all $t\in \Bbb{R}_{++}$.
}

%%%%%%%%%%%%%%%%%%%%%%%%%%%%%%%%%%%%%%%%%%%%%%%%%%%%%%%%%%%%%%
\section*{\centerline{ 2. New Step Size}} \setcounter{section}{2}\setcounter{theorem}{0}
%%%%%%%%%%%%%%%%%%%%%%%%%%%%%%%%%%%%%%%%%%%%%%%%%%%%%%%%%%%%%%

{In this section, we briefly introduce the main optimization problem and state some assumptions. Afterward, we will present the new step size and algorithm. 
\\
We consider the following optimization  problem:  \begin{equation}\label{cost-function}
\min_{x\in \Bbb{R}^d}f(x)=\min_{x\in \Bbb{R}^d} \frac{1}{n}\sum_{i=1}^{n}f_i(x),
\end{equation}
where $f_i :~\Bbb{R}^d\rightarrow \Bbb{R}$ is the loss function for the $i$-th training sample over the variable $x\in\Bbb{R}^d$ and $n$ denotes the number of samples. This minimization problem is central in machine learning. {
Several iterative approaches for solving equation (1) are known \cite{Nocedal}, and SGD is particularly popular when the dimensionality, $n$, is extremely large \cite{robbins1951stochastic,Nemirovski}}. SGD uses a random training sample $i_k \in \{1,2,...,n\}$ to update $x$ using the rule:
\begin{equation}\label{new_point}
    x_{k+1}=x_k-\eta_k\nabla f_{i_k}(x_k),
\end{equation}
in which $\eta_k$ is the step size used in iteration $k$ and $\nabla f_{i_k}(x)$  is the (average) gradient of the loss function(s) \cite{vaswani2019painless}.}

\section*{\centerline{ 2.1. Assumptions}}
{
Let $f: \mathbb{R}^d \rightarrow \mathbb{R}$ be the objective function, and consider the SGD algorithm. We make the following assumptions \cite{li2021second}:
\begin{itemize}
    \item \textbf{ $A1$:}
The function $f: \mathbb{R}^d \rightarrow \mathbb{R}$ is L-smooth, which implies that for all $x$ and $y$ in the domain of $f$, we have:
\begin{equation}
f(y) \leq f(x) + \langle \nabla f(x), y - x \rangle + \frac{L}{2}\|y - x\|^2,
\end{equation}
where, $\nabla f(x)$ denotes the gradient of $f$ at point $x$, and $L$ is a positive constant representing the Lipschitz constant of $f$. %This assumption ensures that $f$ is locally Lipschitz continuous, and it plays a crucial role in establishing the convergence properties of the SGD algorithm.
\item \textbf{$A2$: }
For any iteration $t \in \{1, 2, \ldots, T\}$ of the SGD algorithm, we assume that the expected square norm of the difference between the stochastic gradient $g_t$ and the true gradient $\nabla f(x_t)$ at the current iterate $x_t$ is bounded as follows:
\begin{equation}
\Bbb{E}_t\left[\|g_t - \nabla f(x_t)\|^2\right] \leq \sigma^2,
\end{equation}
where $\sigma^2$ is a positive constant. %This assumption implies that the variance of the stochastic gradients is controlled throughout the optimization process. 
\end{itemize}}

%%%%%%%%%%%%%%%%%%%%%%%%%%%%%%%%%%%%%%%%%%%%%%%%%%%%%%%%
\subsection*{\centerline{ 2.2. The New Step Size}}

In this paper, we address a limitation associated with the $\frac{1}{\sqrt{t}}$ step size, where its value fails to decrease adequately during the final iterations. This characteristic poses a challenge in reaching the optimal point in some problems. To overcome this limitation, we propose a modified step size approach that combines the $\frac{1}{\sqrt{t}}$ function with the logarithmic function, $\ln t$, in an effort to effectively reduce the step size.

The motivation behind incorporating the $\ln t$ function lies in its gradual growth pattern, which enables a more controlled reduction in the step size when compared to the original $\frac{1}{\sqrt{t}}$ step size. By introducing the $\ln t$ function into the formulation, we aim to achieve a more refined and optimized step size throughout the optimization process. In this regard, we define the new step size as:
\begin{equation}\label{new_st}
    \eta_t=\frac{\eta_0}{\sqrt{t}+\ln t},\quad\eta_0\in(0,1],\quad t=1,2,...,T.
\end{equation}
\begin{figure}
  \centering
    \includegraphics[width=0.4\textwidth]{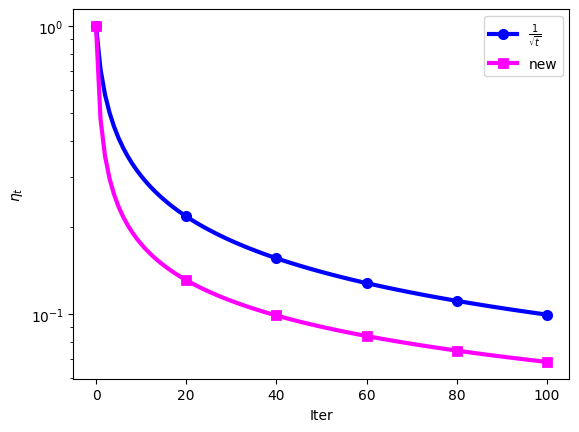}
  \caption{The comparison between the $\frac{1}{\sqrt{t}}$ step size and the modified step size incorporating $\ln t$ function}
  \label{fig:new_step}
\end{figure}
Figure \ref{fig:new_step} illustrates the behavior of two step sizes: the original $\frac{1}{\sqrt{t}}$ step size and the newly proposed step size. The graph visually demonstrates how the new step size consistently selects smaller values, particularly during the final iterations of the SGD algorithm.
\subsection*{\centerline{ 2.3. Algorithm}}
{In this paper, we employ the warm restart Algorithm \ref{alg1} with the  same number of epochs in the inner loop, i.e., $T$.  Algorithm \ref{alg1} is initiated with the provided initial step size $\eta_0$, the number of inner iterations $T$, the number of outer epochs $l$, and the initial point $x_0$. The algorithm consists of the outer and the inner loop. In each inner loop, the SGD with the new step size is executed and the point is updated. 
It is important to note that Algorithm \ref{alg1} was introduced in \cite{li2021second}. It becomes evident that when $l = 1$, Algorithm \ref{alg1} transforms into the SGD algorithm.
\begin{algorithm}
\caption{SGD with warm restarts based on the new step size.}
\label{alg1}
%\begin{algorithmic}[1]   
 {Input:}~Initial step size $\eta_0$, initial point $x_0$, the number of outer and inner iterations, i.e., $l$ and $T$.\\
% { Set:} $T_0=T$\\
\For {$i=1,...,l$}{
 \For {$t=1,...,T$}{
 Run SGD with the new  step size $\frac{\eta_0}{\sqrt{t}+\ln t}$
}
}
\end{algorithm}
}

%%%%%%%%%%%%%%%%%%%%%%%%%%%%%%%%%%%%%%%%%%%%%%%%%%%%%%%%%
\subsection*{\centerline{ 3. Convergence}} \setcounter{section}{3}\setcounter{theorem}{0}

{In this section, we demonstrate that Algorithm \ref{alg1} using the newly proposed step size achieves a convergence rate of $O(\frac{\ln T}{\sqrt{T}})$ for smooth non-convex functions without the PL condition. Note that, the PL condition initially proposed by Polyak \cite{Polyak} and Łojasiewicz \cite{lojasiewicz}, stands as a fundamental cornerstone in demonstrating linear convergence rates for non-convex functions \cite{li2021second}. 
To establish the convergence results, we initially demonstrate an $O(\frac{\ln T}{\sqrt{T}})$ convergence rate for a single outer iteration of Algorithm \ref{alg1}, which corresponds to the SGD algorithm. Subsequently, we extend the proof to encompass $l$ outer iterations. In this regard, we first introduce two preliminary lemmas \cite{li2021second,wang2021convergence}.
}
\begin{lemma}\label{1}
    For the new step size given by (\ref{new_st}), we have:
\begin{equation*}
\sum_{t=1}^T\eta_t\geq \eta_0(\sqrt{T}-1).
\end{equation*}
\end{lemma} 
\begin{proof}
To prove this lemma, we utilize the fact that $\ln t<\sqrt{t}$ for all $t\in[1,~T]$. Thus, we have:
\begin{eqnarray}
\sum_{t=1}^T\eta_t=\eta_0\sum_{t=1}^T\frac{1}{\sqrt{t}+\ln t} \geq \eta_0\sum_{t=1}^T\frac{1}{2\sqrt{t}}=\eta_0\int_1^T\frac{1}{2\sqrt{t}}dt=\eta_0(\sqrt{T}-1),
\end{eqnarray}
where the first inequality is derived from $\ln t\leq \sqrt{t}$ for all $t\geq 1$.
\end{proof}
\begin{lemma}\label{2}
For the new step size given by (\ref{new_st}), we have:
\begin{equation*}
\sum_{t=1}^T\eta_t^2\leq \eta_0^2\ln T.
\end{equation*}

\end{lemma} 
\begin{proof}
To prove this lemma, we use the fact that $\ln t\geq 0$ and $\sqrt{t}+\ln t\geq \sqrt{t}$ for all $t\geq 1$. Hence, we have:
\begin{eqnarray}
\sum_{t=1}^T\eta_t^2=\eta_0^2\sum_{t=1}^T\frac{1}{(\sqrt{t}+\ln t)^2} \leq \eta_0^2\sum_{t=1}^T\frac{1}{t}=\eta_0^2\ln T,
\end{eqnarray}
where the first inequality is obtained from the fact that $\ln t\geq 0$ for all $t\geq 1$.
\end{proof} 

These preliminary lemmas provide important insights and bounds that will be used to establish the convergence results for the modified step size in subsequent sections.

\begin{lemma}[Lemma 7.1 in \cite{wang2021convergence}]\label{pre_proof}
Assuming that $f$ is an $L$-smooth function and Assumption (A2) is satisfied, if  $\eta_t\leq \frac{1}{cL}$, then SGD guarantees:
\begin{equation}\label{for_lemma7}
    \frac{\eta_t}{2}\Bbb{E}[\|\nabla f(x_t)\|^2]\leq \Bbb{E}[f(x_t)]-\Bbb{E}[f(x_{t+1})]+\frac{L\eta^2_t\sigma^2}{2}.
\end{equation}
\end{lemma}
{The following theorem provides $O(\frac{\ln T}{\sqrt{T}})$ rate of convergence for smooth non-convex functions without PL condition.}
\begin{theorem}\label{comp}
Under Assumptions { $A1$ and $A2$,} and for $c>1$, a single outer iteration of Algorithm \ref{alg1}, which corresponds to the SGD algorithm with the new proposed step sizes using $\eta_0=\frac{1}{cL}$ guarantees the following inequality:
\begin{eqnarray}
 \Bbb{E}[\|\nabla f(\bar{x}_T)\|^2]\leq \frac{\ln T}{\sqrt{T}-1}\left[\frac{2Lc\left(f(x_1)-f^*\right)}{\ln T} +\frac{\sigma^2}{Lc}\right]  \nonumber
\end{eqnarray}
where $\bar{x}_T$ is a random iterate drawn from the sequence $\{x_t\}
_{t=1}^T$ with probability {$\mathbb{P}[\bar{x}_T=x_t]=\frac{\eta_t}{\sum_{t=1}^T\eta_t}$.}
\end{theorem}

\begin{proof}
Using the definition of $\bar{x}_T$ and Lemma \ref{pre_proof}, we have:
\begin{eqnarray}
 \Bbb{E}[\|\nabla f(\bar{x}_T)\|^2]=\frac{\eta_t\Bbb{E}[\|\nabla f(x_t)\|^2]}{\sum_{t=1}^T\eta_t}
     &\leq&\frac{2\sum_{t=1}^T\left[\Bbb{E}[f(x_t)]-\Bbb{E}[f(x_{t+1})]\right]}{\sum_{t=1}^T\eta_t}
     +\frac{L\sigma^2\sum_{t=1}^T\eta_t^2}{\sum_{t=1}^T\eta_t}\nonumber\\
      &\leq&\frac{2\left(f(x_1)-f^*\right)}{\sum_{t=1}^T\eta_t}+\frac{L\sigma^2\sum_{t=1}^T\eta_t^2}{\sum_{t=1}^T\eta_t}\nonumber
      \\
      &\leq&\frac{2\left(f(x_1)-f^*\right)}{\eta_0(\sqrt{T}-1)}+\frac{L\sigma^2\eta_0\ln T}{(\sqrt{T}-1)}\nonumber\\
          &=&\frac{2cL}{(\sqrt{T}-1)}\left(f(x_1)-f^*\right)+\frac{\sigma^2\ln T}{Lc(\sqrt{T}-1)},\label{comp1}
\end{eqnarray}
where the third inequality is obtained by using Lemmas \ref{1} and \ref{2}. {The expression (\ref{comp1}) can be rewritten as follows:
\begin{eqnarray}
   \Bbb{E}[\|\nabla f(\bar{x}_T)\|^2]\leq \frac{\ln T}{\sqrt{T}-1}\left[\frac{2Lc\left(f(x_1)-f^*\right)}{\ln T} +\frac{\sigma^2}{Lc}\right].
\end{eqnarray}}
\end{proof}
{
\begin{remark}
    Theorem  \ref{comp} implies that:
    \begin{equation*}
      \Bbb{E}[\|\nabla f(\bar{x}_T)\|^2]\leq O\left(\frac{\ln T}{\sqrt{T}}\right).
    \end{equation*}
\end{remark}
This result demonstrates that a single outer iteration of Algorithm \ref{alg1}, which corresponds to the SGD  based on the new modified step size enjoys an $O(\frac{\ln T}{\sqrt{T}})$ rate of convergence for smooth non-convex functions without the PL condition. Remarkably, this rate of convergence matches that of the traditional $\frac{1}{\sqrt{t}}$ step size \cite{wang2021convergence}.}

Utilizing the outcomes derived from Theorem \ref{comp}, we can now calculate the convergence rate for the warm restart SGD algorithm. 
\begin{corollary}\label{warm}  {(SGD with warm restarts): Under Assumptions A1 and  A2, for a given value of $T$  and $\eta_0=\frac{1}{cL}$, Algorithm \ref{alg1} guarantees the following convergence:}
\begin{equation}
\Bbb{E}\|\nabla f(\bar{x}_T)\|^2\leq \frac{\ln T}{\sqrt{T}-1}\left(\frac{2lcL}{\ln T}\left(f(\tilde{x}_{1})-f^*\right)+\frac{\sigma^2l }{Lc} \right),\nonumber
\end{equation}
where $f(\tilde{x}_1)=\max_i{f(x_{1_i})}$ for $i=1,2,...,l$.
\end{corollary}
\begin{proof}
    {Theorem \ref{comp} is true for all $i=1,2,...,l$. Therefore,  we have:
    \begin{eqnarray*}
  &&\min_i (\Bbb{E}\|\nabla f(\bar{x}_{T})\|^2)\leq \sum_{i=1}^l  \Bbb{E}\|\nabla f(\bar{x}_{T_i})\|^2 \\
  &\leq& \sum_{i=1}^{l}\left(\frac{2cL}{\sqrt{T}-1}\left(f(x_{1_i})-f^*\right)+\frac{\sigma^2\ln T}{Lc(\sqrt{T}-1)}\right)\\
   &\leq& l \max_i\left(\frac{2cL}{\sqrt{T}-1}\left(f(x_{1_i})-f^*\right)+\frac{\sigma^2\ln T}{Lc(\sqrt{T}-1)}\right)\\
   &=&\frac{2lcL}{\sqrt{T}-1}\left(f(\tilde{x}_{1})-f^*\right)+\frac{\sigma^2l \ln T}{Lc(\sqrt{T}-1)}\\
   &=&\frac{\ln T}{\sqrt{T}-1}\left(\frac{2lcL}{\ln T}\left(f(\tilde{x}_{1})-f^*\right)+\frac{\sigma^2l }{Lc} \right)
    \end{eqnarray*}
    in which $f(\tilde{x}_1)=\max_i\{f(x_{1_i})\}$.}
\end{proof}

%%%%%%%%%%%%%%%%%%%%%%%%%%%%%%%%%%%%%%%%%%%%%%%%%%%
\subsection*{\centerline{ 4. Numerical Results}} \setcounter{section}{4}\setcounter{theorem}{0}

In this section, we performed two sets of experiments to assess the effectiveness of our proposed scheme. The first series of experiments involved classifying images on two different datasets: FashionMNIST and CIFAR10. These datasets are commonly used in computer vision research for image classification tasks. The second series of experiments focused on the binary classification of patterns, using five different datasets: a1a, a2a, mushrooms, rcv1, and w1a. These datasets cover a diverse range of patterns and are commonly used in machine learning research for binary classification tasks. To assess the performance of our proposed approach, we compared it with state-of-the-art methods through experimental studies. By conducting these comparisons, we gain insights into the effectiveness of our approach and how it stacks up against existing techniques. Now, let's dive deeper into the mentioned methods, datasets, and the learning model applied for the classification task.
%%%%%%%%%%%%%%%%%%%%%%%%%%%%%%%%%%%%%%%%%%%%%%%%%%%
\subsection*{\centerline{ 4.1. Methods}}
{Here, we conduct a comprehensive comparison study to evaluate various step sizes. We consider the following step sizes:}
\begin{itemize}
    \item $\eta_t=constant$
    \item $\eta_t=\frac{\eta_0}{1+\alpha\sqrt{t}}$
     \item $\eta_t=\frac{\eta_0}{1+\alpha t}$
     \item $\eta_t=\frac{\eta_0}{2}\left(1+\cos\frac{t\pi}{T}\right)$
     \item $\eta_t=\eta_0\left(\frac{1}{1+\alpha(\sqrt{t}+\ln t)}\right)$
\end{itemize}
We have various step size update strategies with the following names: SGD with constant step size, step size with $O(\frac{1}{\sqrt{t}})$ decay, step size with $O(\frac{1}{t})$ decay, cosine step size update, and the new step decay method. The parameter $t$ represents the iteration number of the inner loop, and each outer iteration involves multiple iterations for training on mini-batches.

Additionally, we  compare the results of the newly proposed step decay method with Adam \cite{kingma2014adam}, SGD+Armijo method \cite{vaswani2019painless}, PyTorch’s ReduceLROnPlateau scheduler5 (abbreviated as ReduceLROnPlateau), and stagewise step size. In this comparison, we refer to the points where the step size decreases in the stagewise step decay method as milestones. It's worth noting that since Nesterov momentum is used in all SGD variants, the stagewise step decay method essentially covers the performance of multistage accelerated algorithms (e.g., \cite{aybat2019universally}).
%%%%%%%%%%%%%%%%%%%%%%%%%%%%%%%%%%%%%%%%%%%%%%%%%%%%%%%%%%%%%%
\subsection*{\centerline{ 4.2. Multi-Class Classification using Deep Networks}}

FashionMNIST is a dataset that includes a training set of $50,000$ and a test set of $10,000$ grayscale images. Each image in this dataset has a size of $28*28$ pixels. For the classification task on this dataset, we employed a Convolutional Neural Network (CNN) model. {Let us dissect the architecture of the CNN model we used for this task}. It consists of two convolutional layers. The size of the filter used in each convolutional layer is $5*5$. We applied padding of $2$ to ensure that the spatial dimensions of the output feature maps match the input size. The model also incorporates two max-pooling layers with a kernel size of $2*2$. Max-pooling reduces the spatial dimensions and helps in capturing important features while discarding unnecessary details. To further process the extracted features, the model includes two fully connected layers. Each of these layers has $1024$ hidden nodes. The activation function used for the hidden nodes is the Rectified Linear Unit (ReLU), which helps introduce non-linearity into the model and allows it to learn complex patterns effectively. 

In order to prevent overfitting, a dropout technique is applied with a probability of 0.5 in the hidden layer of the deep model. Dropout randomly sets a fraction of the input units to zero during training, forcing the network to learn robust representations. To evaluate and compare the performance of different algorithms, we utilized the cross-entropy function as the loss function.

The CIFAR10 dataset is composed of $60,000$ color images, each with a size of $32*32$ pixels. These images are divided into $10$ different classes, and each class contains $6,000$ images. The dataset is further split into a training set of $50,000$ images and a test set of $10,000$ images.
During the training process on this dataset, a batch size of $128$ is utilized. This means that each epoch of training comprises $390$ iterations. 
To evaluate the performance of the algorithms on the CIFAR10 dataset, we employed a deep learning architecture known as the $20$-layer Residual Neural Network (ResNet). ResNet was introduced by \cite{he2016deep} and has proven to be highly effective in various computer vision tasks. The loss function used in this model is the cross-entropy loss. 

A grid search was conducted to determine the initial values for parameters $\eta_0$ and $\alpha$ for FashionMNIST and CIFAR10, resulting in $\{0.05,0.15\}$ and $\{0.0253,0.025\}$, respectively. For the remaining step sizes, the initial values from \cite{li2021second} were employed.
%%%%%%%%%%%%%%%%%%%%%%%%%%%%%%%%%%%%%%%%%%%%%%%%%%%%%%%%%%%%%%%%%%%%%%%%%%%

\begin{table*}[t] 
\centering
\begin{tabular}{ |p{1.7cm}|p{2.4cm}|p{2cm}|p{2cm}|p{2cm}|}
\hline
Data set & Dimension (d)   & Training Set Size & Test Set Size   & Kernel Bandwith  \\
  \hline
     a1a  &   $123$ & $24765$ &   $6191$ & $1$\\
 a2a &   $123$  & $24237$& $6060$   & $1$ \\
mushrooms & $112$  & $6499$ & $1625$ & $0.5$\\
rcv1 & $47236$  & $16194$ & $4048$ & $0.25$\\
w1a & $300$  & $37818$ & $9454$ & $1$\\
\hline
\end{tabular}
    \caption{Details for binary classification datasets.}
    \label{table:dataset-details}
\end{table*}
%%%%%%%%%%%%%%%%%%%%%%%%%%%%%%%%%%%%%%%%%%%%%%%%%%%%%%%%%%%%%%%%%%%%%%%%%%%
%%%%%%%%%%%%%%%%%%%%%%%%%%%%%%%%%%%%%%%%%%%%%%%%%%%%%%%%%%%%%
\subsection*{\centerline{ 4.3. Binary Classification with Kernels}}

This series of experiments aims to classify the data into two classes using the Radial Basis Function (RBF) kernel without introducing any regularization techniques. We experiment with five standard datasets: a1a, a2a, mushrooms, rcv1 and w1a from LIBSVM \cite{chang2011libsvm}.  {These datasets have been widely adopted in the machine learning community \cite{Vrbanc}, which allows researchers to compare the performance of various classification algorithms on the same standardized data.}  
To create these datasets, we exclusively utilized the training sets provided by the LIBSVM library \cite{chang2011libsvm}, and performed an $80:20$ split, where $80\%$ of the data was designated as the training set and the remaining $20\%$ was set aside as the test set. 

The a1a dataset is a widely used benchmark dataset in the field of machine learning and data mining. It consists of binary classification tasks where the goal is to predict whether a person's income exceeds $50,000$ based on various attributes such as age, education, marital status, occupation, etc. The a2a dataset is another well-known benchmark dataset that is used for binary classification tasks. It is similar to the a1a dataset in that it focuses on predicting income, but it contains additional attributes and a larger number of instances. 

The mushrooms dataset is a popular dataset used in the field of classification. It contains attributes of various mushrooms, such as cap shape, cap color, odor, gill size, etc., and the target variable is whether the mushroom is edible or poisonous. The rcv1 dataset, also known as Reuters Corpus Volume 1, is a large collection of news articles from Reuters, a major news agency. It consists of over $800,000$ documents categorized into topics such as business, politics, sports, health, etc. The rcv1 dataset is often used for tasks such as text classification, information retrieval, and natural language processing research. The w1a is derived from the web page dataset. It has two categories and $300$ sparse binary keyword attributes. $2,477$ examples are used, among which $72$ examples are positive. Table \ref{table:dataset-details} illustrates the details for each mentioned dataset. 

For all the datasets mentioned, the initial values of parameters $\eta_0$ and $\alpha$ in the new proposed step size, used for binary classification with kernels, are set to $0.05$ and $0.00001$ respectively. On the other hand, the initial values from the study by \cite{li2021second} were utilized for the remaining step sizes.
%%%%%%%%%%%%%%%%%%%%%%%%%%%%%%%%%%%%%%%%%%%%%%%%%%%%%%%%%%%%%
\subsection*{\centerline{ 4.4. Results and Discussion}}

Based on Figure \ref{fig:Deep-Results} and Table \ref{fashion}, the newly suggested step size demonstrates impressive results. In the FashionMnist dataset, it achieves a training loss that is nearly zero, comparable to the performance of well-established methods like SGD+Armijo. Furthermore, it outperforms all other methods in terms of test accuracy. In the CIFAR10 dataset, SGD with the new step size outperforms the previously studied method with a step size of $O(\frac{1}{\sqrt{t}})$, which is considered the best based on both training loss and test accuracy as illustrated in Figure \ref{fig:Deep-Results}. Table \ref{Table_Comparison_results} also emphasises the superiority of the SGD with new step size over the $\frac{1}{\sqrt{t}}$ step decay in both FashionMnist and CIFAR10 datasets. 

Based on the observations made in Figures \ref{fig:RBF-Results1} and  \ref{fig:RBF-Results2}, it can be seen that the implementation of SGD with a new step decay consistently achieves the highest performance in a1a, a2a, mushrooms, rcv1, and w1a datasets. This is evident in terms of both the training loss and test accuracy in all mentioned datasets. Additionally, this method demonstrates a faster convergence to a satisfactory solution compared to the alternative approaches. 
  {Table \ref{Table_Comparison_results2} illustrates that the introduced step size resulted in a reduction of the loss functions by $0.01$, $0.02$, $0.03$, $0.003$, and $0.02$ for the a1a, a2a, mushrooms, rcv1, and w1a datasets, respectively, in comparison to the $\frac{1}{\sqrt{t}}$ step size. Additionally, the new step size enhances the accuracy of the a1a dataset by $0.7\%$ when contrasted with the $\frac{1}{\sqrt{t}}$ step size.}
%%%%%%%%%%%%%%%%%%%%%%%%%%%%%%%%%%%%%%%%%%%%%%%%%%%%%%%%%%%%%
\begin{figure}
  \centering
    \includegraphics[width=1\textwidth]{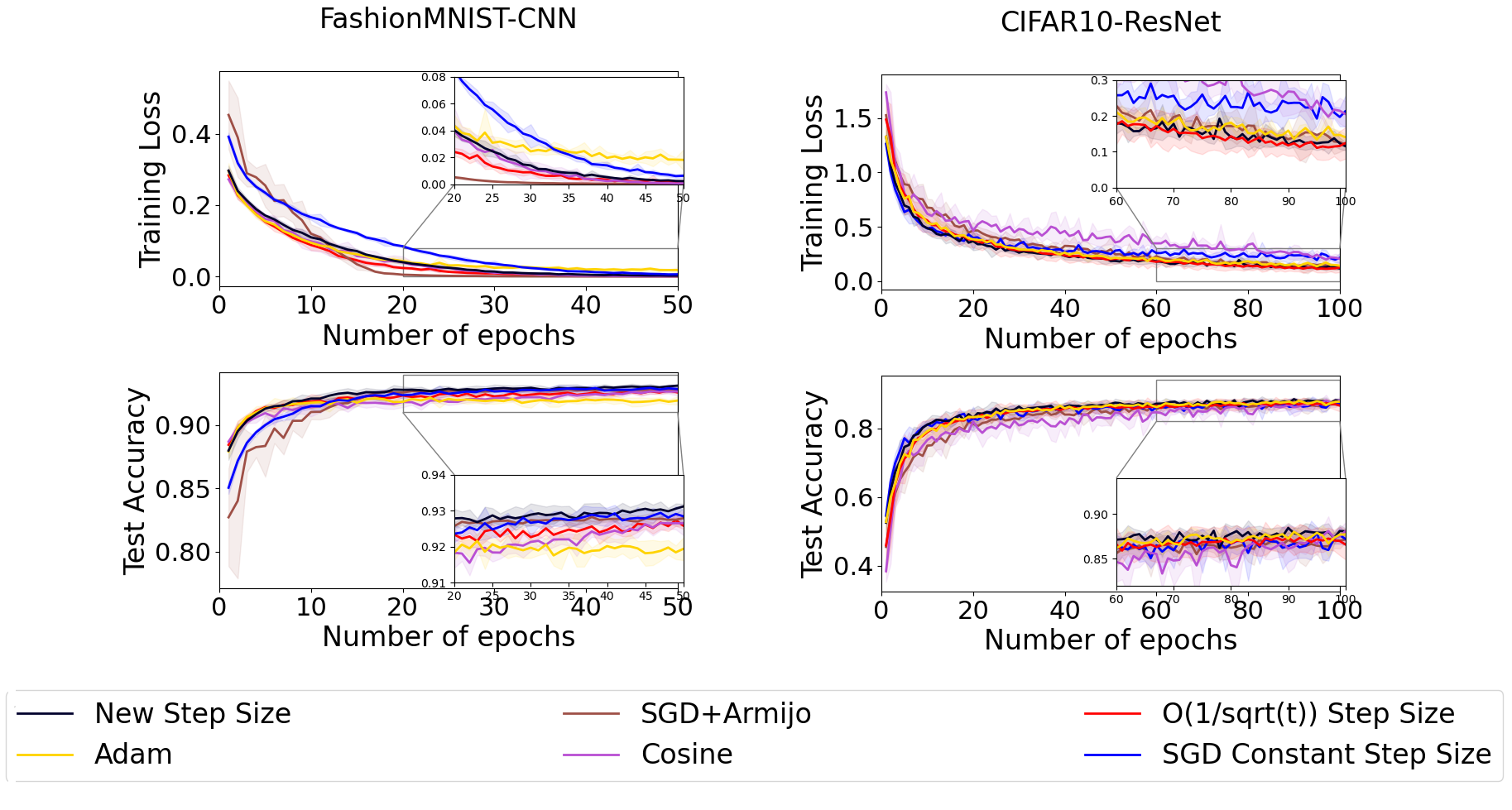}
  \caption{Comparison of new proposed step size and five other step sizes on FashionMNIST and CIFAR10 datasets.}
  \label{fig:Deep-Results}
\end{figure}
%%%%%%%%%%%%%%%%%%%%%%%%%%%%%%%%%%%%%%%%%%%%%%%%%%%%%%%%%%%%%
%%%%%%%%%%%%%%%%%%%%%%%%%%%%%%%%%%%%%%%%%%%%%%%%%%%%%%%%%%%%%
\begin{table*}[t] 
\centering
\begin{tabular}{ |p{4cm}|p{3.5cm}|p{3cm}| }
 %\hline
 %\multirow{2}{4em}{} & \multicolumn{2}{|c|}{FashionMNIST}\\
 \hline
 Methods & Training loss   & Test accuracy \\
  \hline
Constant Step Size &   $0.0007\pm 0.0003$ & $0.9299\pm 0.0016$
 \\
 $O(\frac{1}{\sqrt{t}})$ Step Size & $0.0011\pm 0.0003$ & $0.9261
 \pm 0.0007$\\
 Adam &   $0.0131 \pm 0.0017$  & $0.9166 \pm 0.0019 $\\
 SGD+Armijo & \bf{6.73E-05 $\pm$ 0.00} & $0.9277 \pm 0.0012$ \\
 Cosine step size& $0.0004 \pm 1.1E-05$ & $0.9284 \pm 0.0005$  \\
 New Step Size & ${0.002}  \pm {0.00}$ & $\textbf{0.931} \pm  \textbf{0.00}$ \\
 \hline
\end{tabular}
    \caption{The average final training loss and test accuracy on the FashionMNIST dataset, along with the $95\%$ confidence intervals obtained from $5$ runs starting from different random seeds.}
    \label{fashion}
\end{table*}
%%%%%%%%%%%%%%%%%%%%%%%%%%%%%%%%%%%%%%%%%%%%%%%%%%%%%%%%%%%%%
{\scriptsize
\begin{table}
\centering
\begin{tabular}{ |c|c|c|c|c| }
 \hline
% \multirow{2}{4em}{} & \multicolumn{2}{|c|}{CIFAR10}\\
Step sizes & $\frac{1}{\sqrt{t}}$   & $\frac{1}{\sqrt{t}}$&$\frac{1}{\sqrt{t}+\ln t}$&$\frac{1}{\sqrt{t}+\ln t}$  \\
 \hline
Data set & Training loss   & Test accuracy& Training loss   & Test accuracy  \\
  \hline
 FashionMNIST  &    $0.00 \pm 0.00$ & $0.926\pm 0.00$ &   $\textbf{0.00}  \pm \textbf{0.00}$ & $\textbf{0.93} \pm  \textbf{0.00}$\\
 CIFAR10 &   $0.12 \pm  0.02$  & $0.87 \pm 0.00$& $\textbf{0.08} \pm \textbf{0.01}$   & $\textbf{0.89} \pm \textbf{0.00}$ \\
 \hline
\end{tabular}
    \caption{Performance comparison of $\frac{1}{\sqrt{t}}$ and new step size on FashionMNIST and CIFAR10 datasets.}
    \label{Table_Comparison_results}
\end{table}}
%%%%%%%%%%%%%%%%%%%%%%%%%%%%%%%%%%%%%%%%%%%%%%%%%%%%%%%%%%%%%
%%%%%%%%%%%%%%%%%%%%%%%%%%%%%%%%%%%%%%%%%%%%%%%%%%%%%%%%%%%%%
{\scriptsize
\begin{table}
\centering
\begin{tabular}{ |c|c|c|c|c| }
 \hline
% \multirow{2}{4em}{} & \multicolumn{2}{|c|}{CIFAR10}\\
Step sizes & $\frac{1}{\sqrt{t}}$   & $\frac{1}{\sqrt{t}}$&$\frac{1}{\sqrt{t}+\ln t}$&$\frac{1}{\sqrt{t}+\ln t}$  \\
 \hline
Data set & Training loss   & Test accuracy& Training loss   & Test accuracy  \\
  \hline
 a1a  &    $0.44$ & $0.822$ &   $\textbf{0.43}$ & $\textbf{0.829}$\\
a2a &   $0.39$  & $0.82$& $\textbf{0.37}$   & $0.82$ \\
mushrooms &   $0.57$  & $0.98$& $\textbf{0.54}$   & $0.98$ \\
rcv1 &   $0.531$  & $0.96$& $\textbf{0.528}$ & $0.96$ \\
w1a &   $0.39$  & $1$& $\textbf{0.37}$   & $1$ \\
 \hline
\end{tabular}
    \caption{Performance comparison of $\frac{1}{\sqrt{t}}$ and new step size on datasets for binary classification task.}
    \label{Table_Comparison_results2}
\end{table}}
%%%%%%%%%%%%%%%%%%%%%%%%%%%%%%%%%%%%%%%%%%%%%%%%%%%%%%%%%%%%%
%%%%%%%%%%%%%%%%%%%%%%%%%%%%%%%%%%%%%%%%%%%%%%%%%%%%%%%%%%%%%
\begin{figure}
  \centering
  \includegraphics[width=1\textwidth]{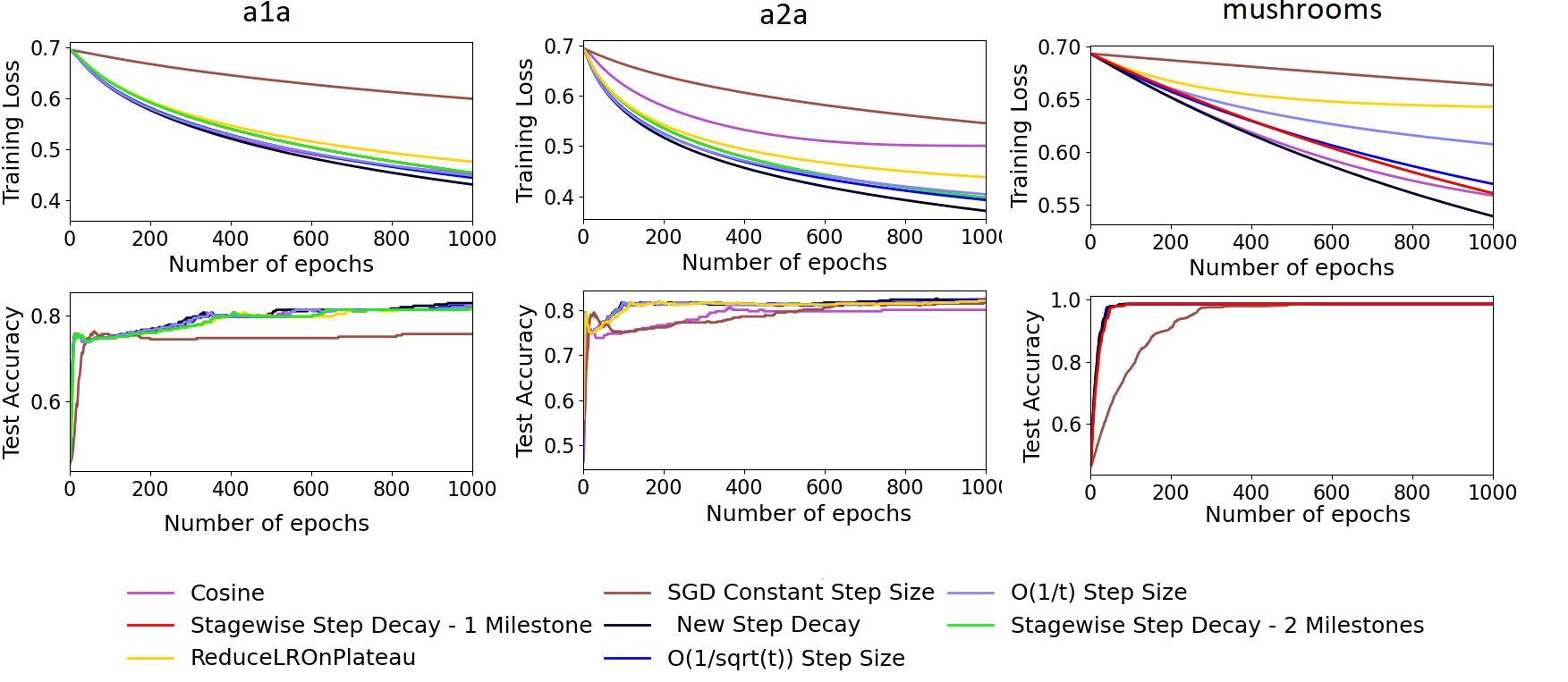}
  \caption{Comparison of new proposed step size and seven other step sizes on a1a, a2a, and mushrooms datasets. }
  \label{fig:RBF-Results1}
\end{figure}
%%%%%%%%%%%%%%%%%%%%%%%%%%%%%%%%%%%%%%%%%%%%%%%%%%%%%%%%%%%%%
%%%%%%%%%%%%%%%%%%%%%%%%%%%%%%%%%%%%%%%%%%%%%%%%%%%%%%%%%%%%%
\begin{figure}
  \centering
    \includegraphics[width=1\textwidth]{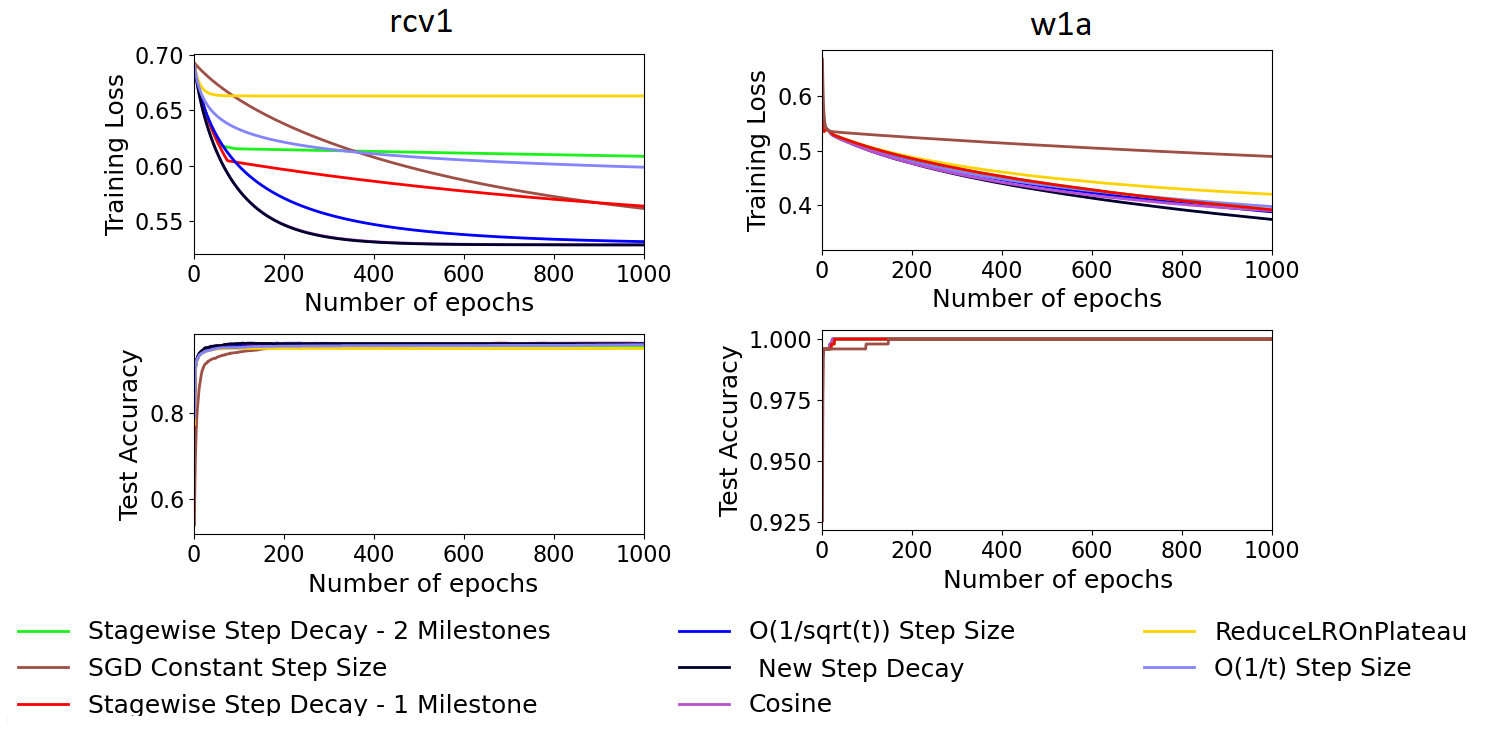}
  \caption{Comparison of new proposed step size and seven other step sizes on rcv1 and w1a datasets.}
  \label{fig:RBF-Results2}
\end{figure}
%%%%%%%%%%%%%%%%%%%%%%%%%%%%%%%%%%%%%%%%%%%%%%%%%%%%%%%%%%%%%
%%%%%%%%%%%%%%%%%%%%%%%%%%%%%%%%%%%%%%%%%%%%%%%%%%%%%%%%%%%%%

\section*{\centerline{ 5. Conclusion}}

This paper introduced a novel approach to enhance the stochastic gradient descent (SGD) algorithm by modifying the decay step size based on $\frac{1}{\sqrt{t}}$. We established a convergence rate of $O(\frac{\ln T}{\sqrt{T}})$ for smooth non-convex functions without the Polyak-Łojasiewicz condition. The numerical experiments conducted on image classification tasks using the FashionMNIST, and  CIFAR10  datasets demonstrated the effectiveness of the proposed approach, with accuracy improvements of  { $0.5\%$},  and $1.4\%$, respectively, over the traditional $\frac{1}{\sqrt{t}}$ step size.  {Furthermore, in the case of binary datasets, the introduced step size exhibited improvements in the loss function by $0.01$, $0.02$, $0.03$, $0.003$, and $0.02$ for the a1a, a2a, mushrooms, rcv1, and w1a datasets, respectively, when compared to the $\frac{1}{\sqrt{t}}$ step size.}

 {As a result of this paper, the combination of the $\frac{1}{\sqrt{t}}$ step size with the $\ln t$ function has led to an enhancement in the efficiency of the $\frac{1}{\sqrt{t}}$ step size. This finding suggests a potential avenue for future research where other step sizes could be similarly combined with suitable functions to boost their efficiency.}

\vskip 3mm
%%%%%%%%%%%%%%%%%%%%%%%%%%%%%%%%%%%%%%%%%%%%%%%%%%%%%%%%%%
\noindent\textbf{Conflicts of Interest.} 
The author affirms that there is no conflict of interest concerning the publication of this manuscript. Furthermore, the authors have diligently addressed ethical concerns, such as plagiarism, informed consent, misconduct, data fabrication and/or falsification, double publication and/or submission, and redundancy.
%%%%%%%%%%%%%%%%%%%%%%%%%%%%%%%%%%%%%%%%%%%%%%%%%%%%%%%%%%%%%

\end{document}